\documentclass{article}
\usepackage{microtype}
\usepackage{subfigure}
\usepackage{balance}
\usepackage{booktabs} 

\usepackage{graphicx}
\usepackage{bmpsize}

\usepackage{hyperref}

\usepackage[accepted]{icml2020}

\icmltitlerunning{Practical applications of metric space magnitude and  weighting vectors}

\usepackage[utf8]{inputenc} 
\usepackage[T1]{fontenc}    
\usepackage{hyperref}       
\usepackage{url}            
\usepackage{booktabs}       
\usepackage{comment}
\usepackage{amsfonts}       
\usepackage{nicefrac}       
\usepackage{microtype}      
\usepackage{bbm}
\usepackage{bm}
\usepackage{mathtools}
\usepackage{mathabx}
\usepackage[algo2e,linesnumbered,algoruled,boxed,lined]{algorithm2e}
\usepackage{amsthm}
\usepackage{graphicx}
\usepackage{subfiles}

\usepackage{array}
\SetKwInOut{Parameter}{parameter}
\newtheorem{theorem}{Theorem}
\newtheorem{lemma}[theorem]{Lemma}

\theoremstyle{definition}
\newtheorem*{example}{Example}
\newtheorem{definition}[theorem]{Definition}

\theoremstyle{remark}
\newtheorem*{remark}{Remark}

\newcommand{\z}{\zeta}
\newcommand{\R}{\mathbb{R}}

\newcommand{\one}{\mathbbm{1}}
\newcommand{\paren}[1]{\left\lparen{#1}\right\rparen}
\newcommand{\norm}[1]{\ensuremath{\left\|{#1}\right\|}}
\newcommand{\abs}[1]{\ensuremath{\left|{#1}\right|}}
\newcommand{\Mag}[1]{\ensuremath{\operatorname{Mag}{\!\paren{#1}}}}
\newcommand{\Vol}[1]{\ensuremath{\operatorname{Vol}{\!\paren{#1}}}}

\mathchardef\mhyphen="2D 

\newcommand\restr[2]{{
  \left.\kern-\nulldelimiterspace 
  #1 
  \vphantom{\big|} 
  \right|_{#2} 
  }}

\title{Practical applications of metric space magnitude and weighting vectors}

\author{ \textbf{Eric Bunch, Daniel Dickinson, Jeffery Kline, Glenn Fung}\\ 
\\ 
American Family Insurance \\
Madison, WI 53783\\
\{ebunch, ddickins, jklin1, gfung\}@amfam.com 
}
\date{}

\begin{document}

\maketitle
\thispagestyle{empty}
\begin{abstract}
Metric space magnitude, an active subject of research in algebraic topology, originally arose in the context of biology, where it was used to represent the effective number of distinct species in an environment. In a more general setting, the magnitude of a metric space is a real number that aims to quantify the effective number of distinct points in the space. The contribution of each point to a metric space’s global magnitude, which is encoded by the {\em weighting vector}, captures much of the underlying geometry of the original metric space.

Surprisingly, when the metric space is Euclidean, the weighting vector also serves as an effective tool for boundary detection. This allows the weighting vector to serve as the foundation of novel algorithms for classic machine learning tasks such as classification, outlier detection and active learning.   We demonstrate, using experiments and comparisons on classic benchmark datasets, the promise of the proposed magnitude and weighting vector-based approaches.

\end{abstract}
\section{Introduction}
\begin{figure}[!t]
\centering
\includegraphics[width=0.4\textwidth]{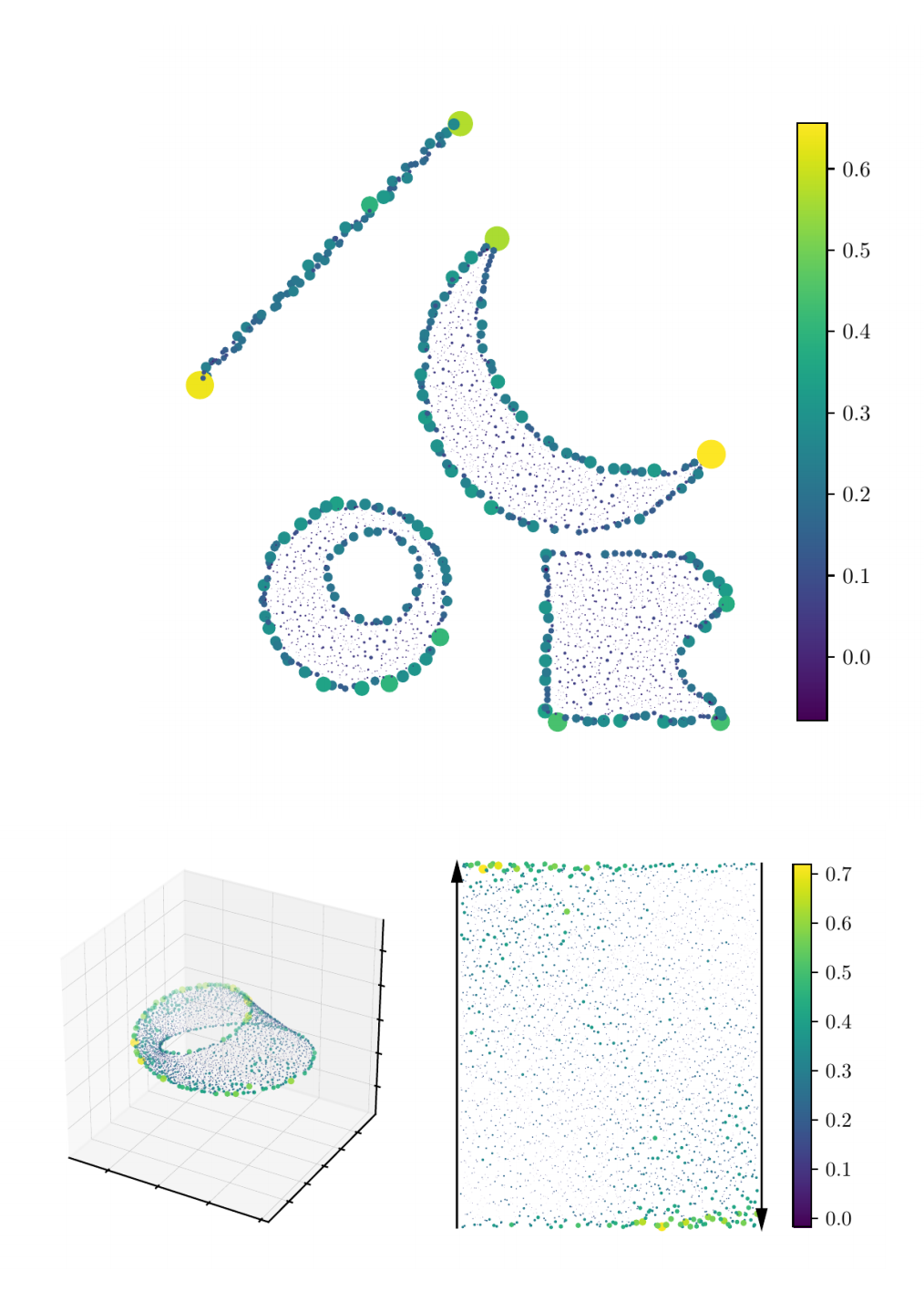}
\caption{A visualization of two weighting vectors.  The set in the top figure is supported within four disjoint components, and they live in $\R^2$. The set in the bottom figure is supported on an embedding of M\"{o}bius strip, and it lives in $\R^3$. In both images, the weight of each point is represented using color and point size.}
\label{fig:numerical_example}
\end{figure}
{\em Magnitude} is a scalar quantity that has meaning for many different kinds of data, and as with other scalar quantities such as rank, diameter, and measure, it has wide applicability, an intuitive interpretation and a solid theoretical foundation. Magnitude has been discovered, and rediscovered multiple times in both practical and theoretical contexts. In this paper, our goal is to apply recent developments drawn from magnitude theory to machine learning, and to empirically demonstrate characteristics of magnitude that, while implicitly described by abstract theoretical results, have not, to our knowledge, been explicitly stated before, nor have they been leveraged for practical purpose.

Informally, magnitude aims to quantify the effective number of points in a space. Our aim is more subtle: we wish to identify \textit{which points} are considered ``effective'' and ``important.''  We do this using the {\em weighting vector}. The weighting vector appears naturally in the definition of magnitude, and we find that the weighting vector, under appropriate conditions, serves as an effective \textit{boundary detector}.   It is this behavior that makes the weighting vector especially well suited for machine learning tasks. 

\subsection{Background, notation and examples}
We now define magnitude and the weighting vector, we present several canonical examples, and we state several central theorems of the field. While our focus is largely on subsets of $\R^n$, we note that the concept {\em magnitude} and {\em weighting vector} can be defined for far more general types of sets.
\begin{definition}
Let $X$ be a finite metric space with metric $d$. Denote the number of points in $X$ by $\abs{X}$. The \textit{similarity matrix} of $X$ is defined to be $\z_X(i, j) := \exp(-d(x_i, x_j))$ for $1 \leq i,  j \leq \abs{X}$. Whenever the inverse of $\z_X$ exists, we define the \textit{weighting vector} of $X$ to be 
    \begin{align*}
        w_X \coloneqq \z_X^{-1}\one,
    \end{align*}
where $\one$ is the $\abs{X} \times 1$ column vector of all ones. The \textit{magnitude} of $X$ is defined to be the quantity
    \begin{align*}
        \Mag{X} \coloneqq \one^Tw_X = \one^T \z_X^{-1} \one.
    \end{align*}
That is, $\Mag{X}$ is the sum of all the entries of the weighting vector $w_X$.
\end{definition}

\begin{example}
When $X$ is a finite subset of Euclidean space, $\z_X$ is a symmetric positive definite matrix [Theorem 2.5.3, \cite{Leinster2013TheMO}]. In particular, $\z_X^{-1}$ is guaranteed to exist. Hence, the weighting vector and magnitude exist for finite subsets of $\R^n$.
\end{example}

\begin{example}
Given an undirected, unweighted graph $G$, one can define a metric space whose points are given by the vertices of $G$, and whose metric is taken to be the length of the shortest path between two vertices. The weighting vector of this metric space is not guaranteed to exist.
\end{example}

\begin{definition}\label{def:infinitesets}
For an arbitrary subset $X \subseteq \R^n$, the magnitude of $X$ is defined as
\begin{align*}
    \Mag{X} = \sup\{ \Mag{Y} \mid Y \text{ is a finite subset of } X \}.
\end{align*}
\end{definition}
\begin{example}
In 1 dimension, and for $t>0$, one has that $\Mag{[0,t]}=1+t/2$. This was shown by Leinster in  \cite{Leinster2013TheMO}. The magnitude of the ball with radius $r$ in $\R^{2n+1}$ is a rational function of $r$, and this was recently demonstrated by Barcel\'{o} and Carbery~\cite{Barcelo18CptEuclid}.
\end{example}

For a finite metric space $(X, d)$, and any $t \in [0, \infty]$, we can define a new metric space $(tX, td)$ in the following way. The points of $tX$ are the same as those of $X$, and the metric $td$ is $d$ scaled by $t$: $td(x, y) \coloneqq t\cdot d(x, y)$. The \textit{magnitude function of $X$} is the map $t \mapsto \Mag{tX}$, and it is well-defined whenever $\z_{tX}$ is invertible. Although the inverse of $\z_{tX}$ may not be defined in general, it has been shown in [Proposition 2.2.6 \cite{Leinster2013TheMO}] that for finite subsets of $\R^n$, the magnitude function is analytic on $(0, \infty)$. We also have the following:
\begin{theorem}[Proposition 2.2.6 \cite{Leinster2013TheMO}]\label{theorem_t_inf_finite_set}
For $X \subset \R^n$ finite, $\lim_{t \rightarrow \infty} \Mag{tX} = \abs{X}$.
\end{theorem}

The above proposition is one of the reasons underlying the informal interpretation of magnitude as quantifying the effective number of points in a space. The following very recent theorem gives a connection between the magnitude of $X \subset \R^n$ and the $n$-volume of $X$.
\begin{theorem}[Theorem 1 \cite{Barcelo18CptEuclid}]\label{theorem_vol}
For $X \subset \R^n$ nonempty and compact, we have
\begin{align*}
    \lim_{t \rightarrow 0^+} \Mag{tX} 
    &= 
    1, \text{\hspace{3pt} and}\\
    \lim_{t \rightarrow \infty}\frac{\Mag{tX}}{t^n} 
    &= 
    \frac{\Vol{X}}{n!\Vol{B_n}},
\end{align*}
\noindent where $B_n \subset \R^n$ is the unit ball.
\end{theorem}

\subsection{Properties of the weighting vector}
The weighting vector plays a central role in the applications that are discussed below, but it is not obvious by inspection of its definition what useful information the individual entries of the weighting vector carries. To provide some intuition about this vector, we now highlight its key features. Our present aim is to convey a qualitative sense of things, so our focus is on numerical examples and basic facts. Note that the ability of the weighting vector to perform boundary detection is more than conjecture: it may be completely explained using harmonic analysis~\cite{folland1999real,meckes2015magnitude}. But for reasons of space and  scope, we limit our focus.

Let $X\subset \R^n$ be a finite set and recall that the weighting vector of $X\subset \R^n$ is defined as $w_X\coloneqq \zeta_{X}^{-1}\one$. This vector is related to the magnitude of $X$ through $\Mag{X}=\one'\zeta_{X}^{-1}\one=\one'w_{X}$. Note that while $X\subset\R^n$, the vector $w_X\subset \R^{\abs{X}}$, i.e., the dimension of the weighting vector is a function of the
size of $X$, and not of the dimension~$n$. Also note that the entries of $w_X$ may be indexed in a canonical way by $x\in X$. We call $w_X(x)$, the weight of $x$.  

Since all operations involved in evaluating $w_X$ are continuous, the weighting vector of a small perturbation of $X$ will approximate the weighting vector of $X$ itself. More precisely, let $X_\epsilon\coloneqq\left\{x+\epsilon \eta_x :x\in X, \norm{\eta_x}\leq 1\right\}$. Then for all $x\in X$, one has $\lim_{\epsilon\rightarrow 0}w_{X_{\epsilon}}(x+\epsilon\eta_x)=w_X(x)$. Since the similarity matrix $\zeta_{X}$ is positive definite, its inverse exists and is also positive definite. Thus, $\one'\zeta_{X}^{-1}\one = \one'w_{X}>0$. Although the average value of the entries of $w_X$ is guaranteed to be positive, it may happen that $w_X(x)<0$ holds for some $x\in X$.  It is currently unknown what, if any, significance to ascribe to the sign of $w_X(x)$.

Let $f:\R^{n}\rightarrow\R^{m}$ be an affine isometry, i.e., $f(x)=Qx+q$ for some $Q\in \R^{m\times n}$ and $q\in \R^m$, and $f$ satisfies $\norm{z-w} = \norm{f(z)-f(w)}$ for all $z,w\in\R^{n}$.
Since $f$ is an isometry, the similarity matrices of $X$ and $f(X)$ agree (i.e., $\zeta_{X}=\zeta_{f(X)}$), and as a result, $w_X(x)=w_{f(X)}(f(x))$. By the Mazur-Ulam theorem, all surjective isometries between normed spaces are necessarily affine. Thus, we have that the weighting vector is invariant under transformation by {any} surjective isometry. More concretely, in Figure~\ref{fig:numerical_example}, the weightings that are displayed are independent of the location and the orientation of the sets. 

When $t>0$ is large, the scaled space $tX$ has the property that all points in it are far from each other.  Thus, one has $\lim_{t\rightarrow\infty}\zeta_{tX}=\lim_{t\rightarrow\infty}\zeta_{tX}^{-1} = I$. By Theorem~\ref{theorem_t_inf_finite_set}, the magnitude function satisfies $\lim_{t\rightarrow\infty}\Mag{tX}=\abs{X}$. Combining these observations, we find that for all $x\in X$,  $\lim_{t\rightarrow\infty} w_{tX}(tx)=1$. The closely-related concept of a {\em scattered space} appears in prior work (Definition 2.1.2, \cite{Leinster2013TheMO}), where under conditions far more general than considered here, it is shown that scattered spaces have well-defined magnitudes, and hence, weighting vectors.

Conversely, when $t>0$ is small, each entry of the similarity matrix $\zeta_{tX}$ is close to 1. In particular, the limiting matrix is the rank-1 matrix $\one\one'$, which does not have an inverse.  However, by Theorem~\ref{theorem_vol}, one has $\lim_{t\rightarrow 0^+}\one'w_{tX}(x)=1$.  Empirically, when $t>0$ is very small, we find that the weights of ``interior points'' of the global space of $X$ are small, while the ``extreme'' points of $X$---especially points that live nearest the convex hull of $X$---are significantly larger.

Finally, we consider weighting vectors of $X\subset\R^n$ that is neither too scattered nor especially concentrated about the origin. As one example, let $X\subset\R^n$ be a regular convex polytope. By a symmetry argument, for all vertices  $x,y\in X$ one has $w_X(x)=w_X(y)$. Thus, modulo a normalizing constant, the weighting vector is completely specified.   Next, consider Figure~\ref{fig:numerical_example}, which displays weightings of two sets that do not have any special symmetry: $X_0\subset \R^2$ which consists of points supported in the union of four disjoint sets, and $X_1\subset\R^3$ which lives on an embedding of the M\"{o}bius strip. Both sets were generated using a uniform random sampling process. In these renderings, every $x\in X_{i}$ has its weight, $w_{X_{i}}(x)$, conveyed both by the marker size and by color, where $i=0,1$. It is clear from these figures that points within the relative interior of some component have low weight, while points in close proximity to some  boundary tend to have larger weight.  It is this empirical observation that leads to the utility of weightings in applications.

We close this section by observing that magnitude, and by extension, weighting vectors, are well-defined on a very general class of sets, including sets that are not necessarily subsets of $\R^n$. It is therefore possible to extend the notion that connects a point's weight and its proximity to a boundary to {\em any} space that has a weighting vector.

\subsection{Related work, paper structure}

An early reference to the concept of magnitude occurs in~\cite{Solow1994}, where it was introduced as a way to measure biological diversity. However, the mathematical motivations were not divulged in this paper. Two decades later, Leinster~\cite{Leinster2013TheMO} placed the magnitude of a metric space within a formal mathematical framework using category theory.  This highly abstract perspective lead to the current era, where it is being explored through many different lenses, including functional analysis \cite{MeckesPosDef, Barcelo18CptEuclid}, harmonic analysis~\cite{meckes2015magnitude} and homology theory \cite{leinster2017magnitude}, where it has been shown to be equivalent to an Euler characteristic.   Much of the prior emphasis has been on a set's  magnitude, and this focus has overshadowed the potential utility of the weighting vector.

Recently, {topological data analysis} has emerged as an approach to the problem of describing the shape of high-dimensional data~\cite{Edelsbrunner2000TopologicalPA, Scopigno04persistencebarcodes, computingPH}. One particularly popular topic within this field is persistent homology \cite{Edelsbrunner2000TopologicalPA}.  Recent efforts have realized magnitude as the Euler characteristic of a homology theory, called magnitude homology \cite{leinster2017magnitude}. It has also been shown that there is a direct relationship between magnitude homology and persistent homology \cite{ otter2018magnitude}; however, the current paper is the first known application of magnitude directly to machine learning.


We now describe the remaining sections of this paper. Section~\ref{sec:properties} presents practical techniques for working with, and computing, the magnitude and weighting vectors of a discrete set. Section~\ref{sec:algorithms} introduces three algorithms that leverage the weighting vector in some essential way. The algorithms perform classification, active learning, and outlier detection.  Section~\ref{sec:results} presents results. We end with concluding remarks in Section~\ref{sec:conclusions}.
\section{Useful properties of magnitude}
\label{sec:properties}
In this section, we offer some techniques that are useful when working with weighting vectors. We discuss how the computation of the weighting vector may be effectively computed by breaking the computation into smaller pieces and ``gluing'' the results together.

\subsection{Inclusion-Exclusion for Weight and Magnitude}
In this section we investigate how to calculate the weighting  vector for a set $Z = X \cup Y$ that is the union of two sets. Here $X, Y,$ and $Z$ are all finite subsets of $\R^n$. To approach this, first we investigate the case when $X$ and $Y$ are disjoint. Then we will look at the case when $Y \subset X$, and show how to calculate either $w_X$ or $w_Y$ when one knows the other. Finally we will arrive at a corrected version of the inclusion-exclusion principle for magnitude, as  well  as the weighting vector.

Before proceeding, we recall the definition of the \textit{Schur complement}.

\begin{definition} Let
$M \coloneqq 
\begin{bmatrix}
A & B \\
C & D
\end{bmatrix}$
be the block matrix where the matrices $A, B, C, D$ are of dimensions $n\times n,  n \times  m, m \times n,$ and $m \times m$ respectively. If $D$ is invertible, then the \textit{Schur complement} of $D$ in $M$ is the $n \times n$ matrix

    \begin{equation*}
        M/D = A - BD^{-1}C.
    \end{equation*}
    
\noindent Similarly, if $A$ is invertible, then the Schur complement of $A$ in $M$ is the $m \times m$ matrix

    \begin{equation*}
        M/A = D - CA^{-1}B.
    \end{equation*}
\end{definition}

Let $\emptyset \neq Y \subset X \subset \R^{n}$ be finite sets. Without loss of generality, we can index the points of $X$ such that the first $\abs{Y}$ of them correspond to those points in $Y$. Then we can see that $\z_X$ can be written as a block matrix
\begin{align}\label{eqn:zeta_block}
\z_X = 
\begin{bmatrix}
\z_Y & \z_{Y, \bar{Y}} \\
\z_{Y, \bar{Y}}^T & \z_{\bar{Y}}
\end{bmatrix},
\end{align}
where $\bar{Y} = X \setminus Y$, and $\z_{Y, \bar{Y}}$ denotes the submatrix of $\z_X$ formed by taking the rows corresponding to $Y$ and columns corresponding to $\bar{Y}$. We can now rewrite the formula $\z_Xw = \one$ using equation \ref{eqn:zeta_block} as the system of equations
\begin{align*}
\z_Y \restr{w_X}{Y} + \z_{Y, \bar{Y}} \restr{w_X}{\bar{Y}} 
&=
\one_Y \\
\z_{Y, \bar{Y}}^T \restr{w_X}{Y} + \z_{\bar{Y}} \restr{w_X}{\bar{Y}} 
&=
\one_{\bar{Y}},
\end{align*}
where $\one_Y$ and $\one_{\bar{Y}}$ are respectively the $\abs{Y} \times 1$ and $\abs{\bar{Y}} \times 1$ column vectors of all ones. Since both $\z_Y$ and $\z_{\bar{Y}}$ are invertible, we can form both of the Schur complements $\z_X/\z_Y$ and $\z_X / \z_{\bar{Y}}$. With these in hand, we can write
\begin{align}
\restr{w_X}{Y}
&=
(\z_X/\z_{\bar{Y}})^{-1} (\one_Y - \z_{Y, \bar{Y}} w_{\bar{Y}}) \label{eqn:disjoint_gluing_Y} \\
\restr{w_X}{\bar{Y}}
&=
(\z_X/\z_{Y})^{-1} (\one_{\bar{Y}} - \z_{Y, \bar{Y}}^T w_{Y}), \label{eqn:disjoint_gluing_Ybar} 
\end{align}
where $w_Y$ and $w_{\bar{Y}}$ are the weight vectors for $Y$ and $\bar{Y}$ respectively, and $\restr{w_X}{Y}$ is the weight vector of $X$, restricted to those indices corresponding to $Y$. Thus if we know $w_Y$ and $w_{\bar{Y}}$, equations \ref{eqn:disjoint_gluing_Y} and \ref{eqn:disjoint_gluing_Ybar} give a way to compute $w_X$.

Next, for finite sets $Y \subset X \subset \R^n$ we wish to calculate either the weight vector $w_X$ or $w_Y$ given the other. 

\begin{definition}\label{def:rho}
For a block matrix $M = \begin{bmatrix} 
A & B \\
C & D 
\end{bmatrix}$,

\noindent with $A$ invertible, define 

\begin{equation*}
    \rho_{MA} = \begin{bmatrix}
A^{-1} B (M/A)^{-1} C A^{-1} & -A^{-1} B (M/A)^{-1} \\
-(M/A)^{-1} C A^{-1} & (M/A)^{-1}
\end{bmatrix}.
\end{equation*}
\end{definition}

\noindent Now recall that for a block matrix $M$ as in Definition \ref{def:rho}, we have that

\begin{equation}\label{eqn:M_A_rho}
    M^{-1}
    =
    \begin{bmatrix}
    A^{-1} & 0 \\
    0 &0 \\
    \end{bmatrix}
    + \rho_{MA}.
\end{equation}












\begin{definition}
For $Y \subseteq X \subset \R^n$ finite sets, assume $\z_X$ is in block matrix format as in Equation \ref{eqn:zeta_block}. Define the matrix
\begin{equation*}
\rho_{XY} = \rho_{\z_X \z_Y}
\end{equation*}
\noindent where $\rho_{XY}$ is taken to be the zero matrix when $Y  = X$, and $\rho_{XY}$ is taken to be $\z_X$ when $Y = \emptyset$.
\end{definition}

\begin{lemma}\label{lem:weight_vector_subset}
For finite sets $Y \subset X \subset \R^n$, let $P_{XY}$ be a permutation matrix such that

\begin{equation*}
    P_{XY} \z_X P_{XY} = \begin{bmatrix}
    \z_Y & \z_{Y \bar{Y}} \\
    \z_{Y \bar{Y}}^T & \z_{\bar{Y}}.
    \end{bmatrix}
\end{equation*}
\noindent Then we have
\begin{align*}
w_X
&=
P_{XY}\begin{bmatrix}
w_Y \\
0
\end{bmatrix}
 + 
 P_{XY}\rho_{XY}\one,
\text{ \hspace{3pt} and} \\
 \Mag{X} 
 &= 
 \Mag{Y} 
 + 
 \one^T\rho_{XY}\one.
\end{align*}
\end{lemma}

\begin{proof}
This follows by setting $M = P_{XY}\z_{X}P_{XY}$, employing Equation \ref{eqn:M_A_rho}, and multiplying on the right by $\one$.
\end{proof}



\noindent We can now calculate the weight vector of $X \cup Y$ where $X$ and $Y$ are not necessarily disjoint. This can be viewed as a corrected inclusion-exclusion principle for weight vectors as well as magnitude.

\begin{theorem}
For finite sets $X, Y \subset \R^n$, set $Z = X \cup Y$. Then we have

\begin{align*}
    w_{Z} 
    &=
    P_{Z X}\left( \begin{bmatrix}
    w_X \\ 0
    \end{bmatrix} 
    + 
    \rho_{Z X}\one \right) 
    +
    P_{Z Y} \left( \begin{bmatrix}
    w_Y \\ 0
    \end{bmatrix} 
    + 
    \rho_{Z Y} \one \right) \\
    &- 
    P_{Z X \cap Y} \left( \begin{bmatrix}
    w_{X \cap Y} \\ 0
    \end{bmatrix} 
    - 
    \rho_{Z X \cap Y}\one \right), \text{\hspace{3pt} and} \\
    \Mag{Z}
    &=
    \Mag{X}
    + 
    \Mag{Y}
    -
    \Mag{X \cap Y} \\
    &+
    \one^T \rho_{Z X} \one  
    +
    \one^T \rho_{Z Y} \one 
    -
    \one^T \rho_{Z X \cap Y} \one.
\end{align*}



\end{theorem}

\begin{proof}
This follows by applying Lemma \ref{lem:weight_vector_subset} to each subset considered, e.g.

\begin{equation*}
    w_Z 
    = 
    P_{ZX}\begin{bmatrix}
    w_X \\ 0
    \end{bmatrix}
    +
    P_{ZX}\rho_{ZX}\one.\qedhere
\end{equation*}
\end{proof}




\subsection{Numerical Considerations}\label{numerical_considerations}

In the setting where we have finite sets $Y \subset X \subset \R^n$, and we have calculated $w_Y$, we can calculate $w_X$ without having to invert the entire matrix $\z_X$ using Corollary \ref{lem:weight_vector_subset}. Since 

    \begin{equation*}
        w_X = \begin{bmatrix}
        w_Y \\ 0
        \end{bmatrix} + \rho_{XY}\one,
    \end{equation*}
    
\noindent we only need to invert the matrices $\z_Y$--which we are assuming we have already done--and $\z_X/\z_Y$, which is an $\abs{X\setminus Y} \times \abs{X \setminus Y}$ matrix. Then all the matrix products must be performed in the block matrix formulation of $\rho_{XY}$. In particular, for the case when we are adding a single point to the set $Y$, $\z_X/\z_Y$ is a scalar, and the products needed to form $\rho_{XY}$ are matrix-vector products. This will be used in the sequel to perform more efficient inference of the machine learning classifier.

\section{Algorithms}\label{sec:algorithms}
In this section we give details on how one may use weighting vectors and magnitude for a number of typical machine learning tasks.

\subsection{Magnitude as a classifier}\label{sec:wt_vector}

In this section, we develop an algorithm that uses metric space magnitude for a machine learning classification task. In a classification task, we are given a set $X$ of $m$ training examples in $\R^n$, $x_i \in X \subset \R^n$, $i \in \{1,2,...,m\}$. Each $x_i$ has an associated label, $l(x_i) \in L$, which is an element of a finite set of possible labels, $\abs{L} = k < \infty$. Given an unlabeled new point, $x^{\prime} \in \R^n$, we seek to assign it an associated label $l(x^\prime) \in L$.

Classification is fundamentally a task of finding or defining boundaries, thus because the weight vector can serve as a boundary detector, it is a natural fit for the task. In a classification task, we are working with finite sets of points $X$, so the term ``boundary'' is not well-defined in the mathematical sense. This prompts the following convention: A point $x_i \in X \subset \R^n$ with $\abs{X} < \infty$ is in the interior of $X$ if its weight value is sufficiently small (where ``sufficient'' is context-specific). Two points regarding our convention are worth mentioning. First, for convex sets, Definition \ref{def:infinitesets} ensures that our convention matches with intuition on finite subsets that are sampled sufficiently densely, as the points with small weight all lie near the interior. Second, we can't distinguish between a point near the boundary of a set and one on the exterior of a set, as both will have relatively high weight. However, as discussed below, using our convention of interior points will be sufficient for use in a classification setting.

The weight of a point, and therefore our notion of interior points of a finite set captures global information, as it depends on all other points in the set. By changing other points in the set $X$, but leaving $x_i$ fixed, it's weight $w_i$ changes; the difference in the weight of a point relative to changes in the set is the key part of the classification algorithm.

Let $L = \{L_1,L_2,...,L_k\}$ be the set of labels, and $X_i = \{x \in X \mid l(x) = L_i\}$. If $x^\prime$ is an unlabeled point, the logic proceeds as follows. For each label $L_i$, compute $w_i^\prime$, the weight of $x^\prime$ in the set $\{x^\prime\} \cup X_i$. If weight vectors and inverse matrices for each $X_i$ are computed in advance and cached, by the discussion in \ref{numerical_considerations}, each $w_i^\prime$ only requires matrix multiplication of order $\abs{X_i}$ and inverting a $1 \times 1$ matrix. Intuitively, if $w_i^\prime$ has a low value, it likely is an interior point of $X_i$ and therefore $l(x^\prime) = L_i$ is appropriate. However, if $w_i^\prime$ has a high value, it is likely not on the interior of $X_i$, so another label is more appropriate. Figure \ref{classifier_example} shows an example.

 If the classes are well-balanced and have similar underlying distributions, using the original metric space for each class is appropriate as the values of the $w_i^\prime$ will be similarly scaled and directly comparable. When the classes are imbalanced or have different underlying distributions, that assumption may not be appropriate, as the values of $w_i^\prime$ will not necessarily be comparable. We overcome this potential limitation by introducing a parameter $t_i$ for each $L_i$ that is used to scale distances when calculating $w_i^\prime$, that is we perform all operations related to $L_i$ in the metric space $(t_i X,t_i d)$. Optimal $t_i$ can be tuned during training for example using grid search and cross-validation. For simplicity and readability, we omit $t$ from the basic version in algorithm \ref{algo:classification}. 

We can further ensure consistency between the $w_i^\prime$ for different $i$ by introducing a function $\text{SCALE}_i:(\R,\R^{\abs{X_i}}) \rightarrow \R$, which serves to normalize $w_i^\prime$ relative to weights of other points with label $L_i$. Taking $w_j^i$ to be the weight of $x_j \in X_i$, some examples of possible functions are absolute value, $\text{SCALE}_i(w_i^\prime,\{w_j^i \mid x_j\in X_i\}) = \abs{w_i^\prime}$, and percentile $\text{SCALE}_i(w_i^\prime,\{w_j^i \mid x_j\in X_i\}) = \frac{\abs{\{w_j^i \mid w_j^i <= w_i^\prime\}}}{\abs{X_i}}$.

We select a class label using a function $\text{DECIDE}$ which operates on the $w_i^\prime$ after they have been scaled using $\text{SCALE}_i$. Letting $\text{S}_i(w_i^\prime)$ denote $\text{SCALE}_i(w_i^\prime,\{w_j^i \mid x_j\in X_i\})$, an example of $\text{DECIDE}$ is argmax, $\text{SCALE}(\text{S}_1(w_1^\prime),\text{S}_2(w_2^\prime),...,\text{S}_k(w_k^\prime)) = i $ where $\text{S}_i(w_i^\prime)>\text{S}_j(w_j^\prime)$ for all $j\neq i$ By allowing $\text{DECIDE}$ to accept one additional threshold parameter, however, the algorithm can account for previously unseen classes as follows. If all $w_i^\prime$ are above the threshold parameter, it is likely the point is far from any of the labeled points, and thus from an unseen class, so it is assigned \texttt{NULL}. Otherwise, apply the decision function as described above. Note that for simplicity and readability, we omit the threshold parameter from the basic version presented in algorithm \ref{algo:classification}. 

\begin{figure}[ht]
\centering
  \includegraphics[scale=0.4]{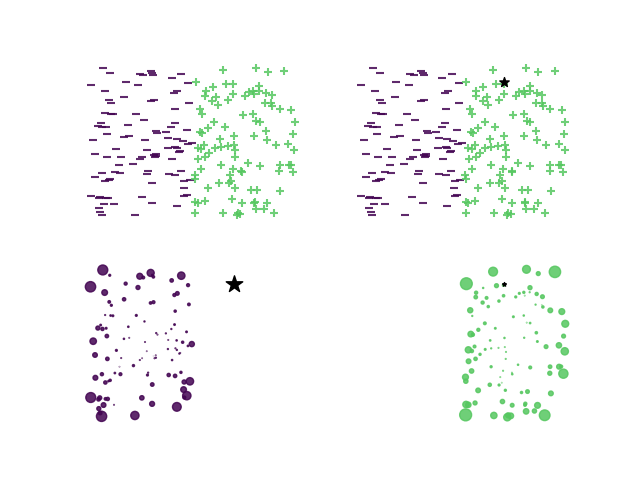}
  \caption{Upper left: Training data $X$, with $L_0 = -$ and $L_1 = +$. Upper right: $X \cup x^\prime$, with a star denoting $x^\prime$. Lower left: $\{x^\prime\} \cup X_0$, with $w_0^\prime = 0.517$, and the sizes of markers indicate weight. Lower right: $\{x^\prime\} \cup X_1$, with $w_1^\prime = 0.026$, and the sizes of markers indicate weight.}
  \label{classifier_example}
\end{figure}

\begin{algorithm}[H]
\caption{Classification via weighting vector}
{\small
\begin{algorithmic}
\INPUT Data set $X$, $L = \{L_1,L_2,...,L_k\}$ labels, function $\text{DECIDE}:\R^k \rightarrow \{1,2,...,k\}$, function $\text{SCALE}_i:(\R,\R^{\abs{X_i}}) \rightarrow \R$ for each $i\in\{1,2,...,k\}$
\INPUT unlabeled point $x^\prime$
\STATE $p$ = []
\FOR {$i \in \{1,2,...,k\}$}
\STATE $Y = \{x^\prime\} \cup X_i$
\STATE $w_i^\prime = w_Y(x^\prime)$
\STATE $w=\text{SCALE}_i(w_i^\prime,W_{X_i})$
\STATE $p.$append($w$)
\ENDFOR
\STATE let $j = \text{DECIDE}(p)$
\OUTPUT $L_j$
\end{algorithmic} 
}
\label{algo:classification}
\end{algorithm}






\subsection{Magnitude for active learning}
\label{sec:al}
\begin{figure}[ht]
 \centering
  \includegraphics[scale=0.18]{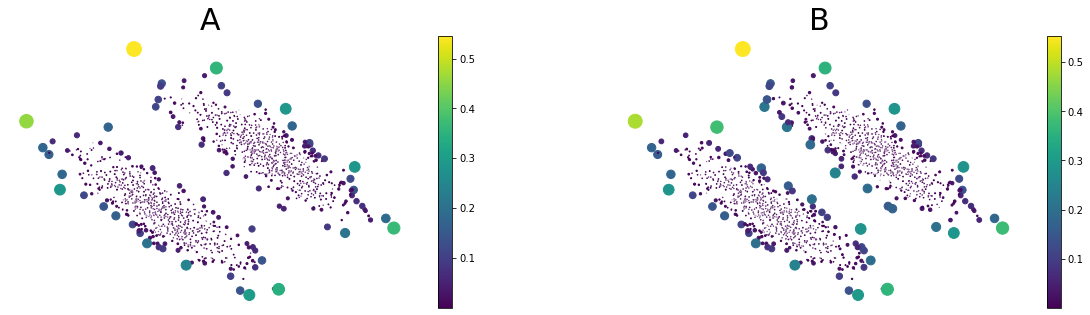}
 \caption{Magnitude for active learning example. A is weight of whole data set. B is weight of each class separately.}
  \label{active_learning_example2}
 \end{figure}
Next, we will describe how we can use magnitude and the weight vector to define a query strategy for an active learning algorithm. As stated in \cite{activeLearningSurvey:Burr}, ``The key idea behind active learning is that a machine learning algorithm can achieve greater accuracy with fewer training labels if it is allowed to choose the data from which it learns.''

Approaches that minimize the number of human feedback needed to train machine learning models have sparked renewed interested due to the cost of labeling and the the fact that recent deep-learning-based approaches need handle large amounts of training data to achieve optimal performance. 
Let $\mathcal{L}$ (the labeled dataset) and $\mathcal{U}$ the (unlabeled dataset) be two subsets of the available pool of training data $X$, with $X=\mathcal{U} \cup \mathcal{L}$ and $\mathcal{U} \cap \mathcal{L}=\emptyset$. An iteration of the algorithm will pick some points in  $\mathcal{U}$ to be labeled by an oracle (transferring them to  $\mathcal{L}$). The current model will be then updated using the new updated dataset  $\mathcal{L}$ and its corresponding labels. 

For simplicity we will state the algorithm for a binary classification problem i.e. when $L=\{L_{0},L_{1}\}$, however it can be trivially extended to a multi-class problem. 

The intuition behind the algorithm is simple: at each iteration $i$, we assign every training data point to one of the sets  $\tilde{X_0}$ or $\tilde{X_1}$ according to its predicted label by the current classifier $f_i$. We will calculate the corresponding weight vectors $w_{\tilde{X_{0}}}$ and $w_{\tilde{X_{1}}}$. Then, we choose to label (submit to an oracle for labeling)  the point with the minimum value (interior point) and the with the maximum value (likely to be in the boundary) for both sets $\tilde{X_0}$ and $\tilde{X_1}$. By choosing  this way we are aiming to: (a) reinforce, validate and refine high confidence classifier information (labels) acquired in prior iterations (exploitation) and (b) to acquire labels in the predicted class boundaries where our classifier confidence is potentially lower (exploration). The proposed active learning algorithm is stated below.

\begin{algorithm}
\label{alg:al}
\caption{Active learning via weighting vector}
{\small
\begin{algorithmic}
\INPUT Data set $X$, 
\STATE $\mathcal{L}=\emptyset$; $\mathcal{U}=X$
\STATE initialize $\mathcal{L}$ ;  $\mathcal{U}=X-\mathcal{L}$;  with it's corresponding $\mathcal{Y_{\mathcal{L}}}$
\STATE $f={\bf train\_classifier}(\mathcal{L}$,$\mathcal{Y_{\mathcal{L}}}$)
\WHILE{(not converged) {\bf or} (labeling budget not reached)}
    \STATE $\tilde{X_i} = \{ x \in X \mid f(x)  = i \}$ for $i = 0, 1$.
    \STATE calculate weighting vectors $w_{\tilde{X_{i}}}$
    \STATE $Q_{\min,i}=\underset{\mathcal{U}}{\arg\min}  \ {\left| w_{\tilde{X_{i}}}\right|}$ for $i = 0, 1$
    \STATE $Q_{\max,i}=\underset{\mathcal{U}}{\arg\max}  \ {\left|w_{\tilde{X_{i}}}\right|}$ for $i = 0, 1$
    \STATE $\mathcal{Y_{\mathcal{Q}}}$=query\_labels($Q_{\min,0}$,$Q_{\max,0}$,$Q_{\min,1}$,$Q_{\max,1}$)
    \STATE $\mathcal{L}=\mathcal{L} \cup \{Q_{\min,0}$,$Q_{\max,0}$,$Q_{\min,1}$,$Q_{\max,1}\}$
    \STATE $\mathcal{Y_{\mathcal{L}}}=\mathcal{Y_{\mathcal{L}}} \cup  \mathcal{Y_{\mathcal{Q}}}$
    \STATE $\mathcal{U}=X-\mathcal{L}$;
    \STATE $f={\bf train\_classifier}(\mathcal{L},\mathcal{Y_{\mathcal{L}}}$)
\ENDWHILE   
\OUTPUT $f$
\end{algorithmic}
}
\label{algo:activelearning}
\end{algorithm}
Where ${\left| w_{\tilde{X_{i}}}\right|}$ denotes all components of the vector $ w_{\tilde{X_{i}}}$ in absolute value.
We present some numerical experiments in Section \ref{sec:results}.
\subsection{Magnitude for Outlier Detection}

In this section we give an algorithm that uses the values of the weight vector of a set to find outliers in a dataset. As we have seen, the weight vector serves as a boundary detector for a data set. But if the boundary is not well defined because there are outlier data points, we can use the weight to mark points as outliers. Suppose we have a data set $X \subset \R^n$, and wish to determine if a new point $x \in \R^n$ should be considered an outlier with respect to $X$. By looking at the value $\gamma_{Xx} := \one^T\rho_{X\cup\{x\}X}\one = \Mag{X\cup \{x\}} - \Mag{X}$, we can see if adding $x$ increased the magnitude substantially, thereby greatly extending the "border" of $X$. By Lemma 3.1.3 in \cite{Leinster2013TheMO} we have that $0 \leq \gamma_{Xx}$.

Care must be taken, however; both the points on the boundary of the data set, and the outlier points will have high weight relative to the interior of the data set. Thus we collect all the points in $X$ whose weight is below a threshold (here we take median weight plus 1.5 times standard deviation), and denote this subset as $X_{in}$, the \textit{inliers}. The points of $X$ not in $X_{in}$ we call  \textit{outlier candidates}, and denote as $X_{out}$. Next, for each $x \in X_{out}$ with  $\gamma_{Xx}$ less than a user-defined threshold $0 \leq \tau$, we move from $X_{out}$ to $X_{in}$. Then we have our final decomposition of the data set into inliers and outliers: $X = X_{in} \cup X_{out}$. We record this algorithm in Algorithm \ref{algo:outlier}.

In Figure \ref{fig:outlier_example_synthetic} we have the results of this algorithm using synthetic data. Inlier data was generated from two Gaussian distributions, and outliers were drawn from a uniform distribution.
\begin{remark}
It can be noted that the \texttt{NULL} class prediction algorithm described in Section \ref{sec:wt_vector} can be viewed as a type of online outlier detection algorithm. If the same paradigm is used when there is a single class, we obtain an outlier detection algorithm that is trained on data that only contains inliers.
\end{remark}
\begin{algorithm}
\caption{Outlier detection via magnitude}
\begin{algorithmic}
\INPUT dataset $X$, threshold $\tau$
\STATE $X_{in} = \{ x \in X \mid \text{abs}(w_X(x)) < \text{median}(w_X) + 1.5\text{std}(w_X) \}$
\STATE $X_{out} = X \setminus X_{in}$
\FOR {$x \in X_{out}$}
\IF {$\gamma_{Xx} < \tau$}
\STATE $X_{in} \gets x$
\ENDIF
\ENDFOR
\end{algorithmic} 
\label{algo:outlier}
\end{algorithm}
\begin{figure}
\centering
  \includegraphics[scale=0.25]{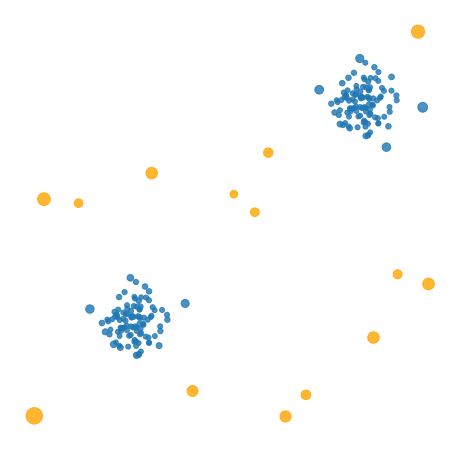}
  \caption{Outlier detection for synthetic data, $\tau = 0.2$}
  \label{fig:outlier_example_synthetic}
\end{figure}
\section{Results}\label{sec:results}

\subsection{Classification Experiments}

To test the classification algorithm, we ran a set of ten experiments across 5 classic benchmark datasets from the UCI repository, a synthetic two-dimensional checkerboard dataset, as well as the scikit-learn digit and iris datasets, and  multiple classifiers. Each experiment consisted of using a random stratified splitting method to partition the the data set into a training set consisting of 70\% of the data, and a testing set consisting of the remaining 30\%. The classifiers were trained without fine-tuning any parameters; the basic algorithm presented in \ref{algo:classification} with $\text{ARGMAX}$ for $\text{DECIDE}$ and absolute value for $\text{SCALE}_i$, and the defaults in scikit-learn \cite{scikit-learn} for all parameters in the other algorithms. Table \ref{tbl:clfn_exp_res} records the average and standard deviation of the accuracy on the testing dataset for all classifiers.

\begin{remark}
Our model performed quite similarly to $k$-nearest neighbors in our experiments, which is quite remarkable given the dramatic differences between the algorithms. We also note the promise it implies: our initial attempt at using the boundary detection properties of the weighting vector in a machine learning setting have matched the performance of a well-established and widely-used model. We believe this will be improved upon and expanded as techniques using the weighting vector are adopted more widely.
\end{remark}



\begin{table*}[t]
\caption{}
\centering
{\small
\begin{tabular}{llllll}
\toprule
dataset &  K-Neighbors & Logistic Reg. & Rand. Forest &          SVM &       Weight \\
\midrule
2-d checkerboard &  0.92 $\pm$ 0.02 &   0.51 $\pm$ 0.04 &   {\bf 0.94 $\pm$ 0.01} &  0.62 $\pm$ 0.04 &  0.92 $\pm$ 0.01 \\
clevedata.mat    &  0.82 $\pm$ 0.04 &   {\bf 0.85 $\pm$ 0.02} &   0.82 $\pm$ 0.03 &  0.84 $\pm$ 0.03 &  0.84 $\pm$ 0.03 \\
dimdata.mat      &  0.94 $\pm$ 0.01 &   0.95 $\pm$ 0.01 &   0.95 $\pm$ 0.00 &  {\bf 0.96 $\pm$ 0.00} &  0.93 $\pm$ 0.01 \\
housingdata.mat  &  0.87 $\pm$ 0.02 &   0.87 $\pm$ 0.03 &   0.87 $\pm$ 0.02 &  0.87 $\pm$ 0.03 &  0.87 $\pm$ 0.02 \\
ionodata.mat     &  0.84 $\pm$ 0.05 &   0.89 $\pm$ 0.02 &   0.94 $\pm$ 0.02 &  {\bf 0.95 $\pm$ 0.02} &  0.81 $\pm$ 0.08 \\
iris             &  0.94 $\pm$ 0.04 &   0.87 $\pm$ 0.05 &   0.94 $\pm$ 0.04 &  {\bf 0.96 $\pm$ 0.03} &  0.85 $\pm$ 0.13 \\
sklearn digits   &  0.97 $\pm$ 0.01 &   0.96 $\pm$ 0.01 &   0.95 $\pm$ 0.01 &  {\bf 0.98 $\pm$ 0.01} &  0.97 $\pm$ 0.00 \\
ticdata.mat      &  0.85 $\pm$ 0.02 &   0.69 $\pm$ 0.03 &   {\bf 0.93 $\pm$ 0.02} &  0.88 $\pm$ 0.02 &  0.78 $\pm$ 0.03 \\
\bottomrule
\end{tabular}
}
\end{table*}
\label{tbl:clfn_exp_res}
To demonstrate the \texttt{NULL} class label capabilities, we trained the magnitude classifier on examples of six and nine from the scikit-learn digits dataset, then predicted on images of ones, sixes, and nines. The confusion matrix with a \texttt{NULL} class threshold of $1-10^{-11}$ is shown in table \ref{tab:null_class_bset}.

\begin{table}
{\small
\centering
\begin{tabular}{lrrr}
\toprule
{} &  null &   6 &   9 \\
\midrule
null &    53 &   0 &   1 \\
6    &     1 &  53 &   0 \\
9    &     1 &   0 &  54 \\
\bottomrule
\label{tab:null_class_bset}
\end{tabular}
\caption{Confusion matrix for classifier with \texttt{NULL} class.}
}
\end{table}

\subsection{Active learning Experiments}
\begin{figure*}[ht!]
{\small
\label{fig:al_results}
{\centering
  \includegraphics[scale=0.33]{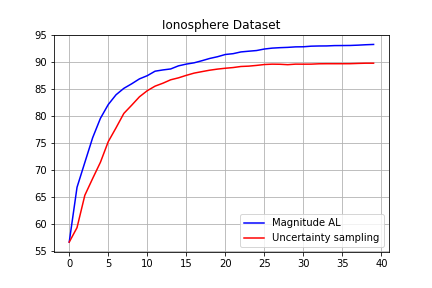}
  \includegraphics[scale=0.33]{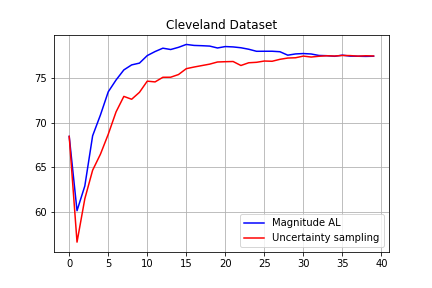}
  \includegraphics[scale=0.33]{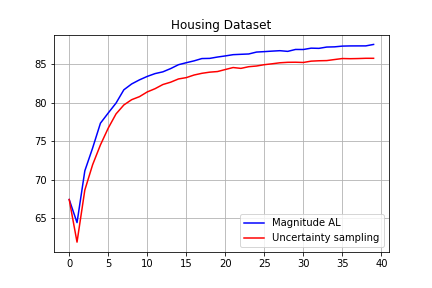}
  \includegraphics[scale=0.33]{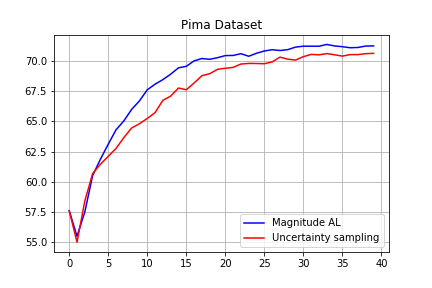}
  \includegraphics[scale=0.33]{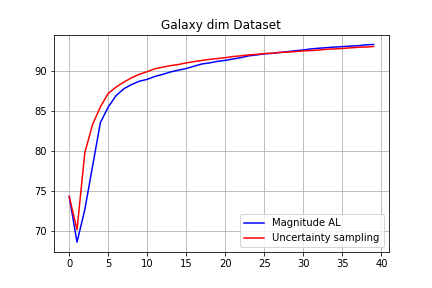}
  \includegraphics[scale=0.33]{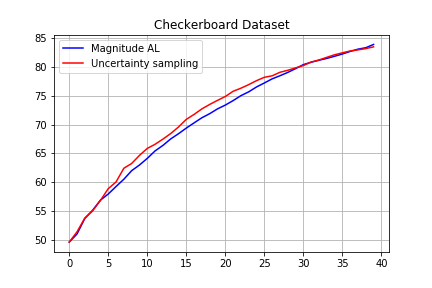}
  \caption{Active learning results comparing the weighting vector query strategy vs the uncertainty sampling strategy. Average over 100 runs}
  \label{al_experiments}
  }
  }
\end{figure*}
In order to assess the effectiveness of the weighting-vector-based active learning (AL) algorithm proposed in Section \ref{sec:al}, we compared Algorithm \ref{alg:al} to the simplest but highly effective and most commonly used query AL framework: uncertainty sampling \cite{lewis94}. In this framework, the AL algorithm queries the instances for which it is least certain about how to label (i.e. for many algorithms $p(label\|x) \approx 0.5$ or where the decision function is close to $0$). For simplicity we used a kernelized Ridge regression model \cite{cristianini2000} (also refer as to LS-SVM \cite{lssvm} or proximal SVM \cite{ProximalSVM}). Laplacian kernels were used both as magnitude to calculate the weighting vector and as classification kernel ($k(x, y) = \exp( -\gamma \| x-y \|_1)$ with $\gamma=0.1$. At each iteration of Algorithm \ref{alg:al} the classifier learned after obtained labels from the oracle has the form $f(x)=K(x,\mathcal{L})'w-w_0$, where $w_0$ is the bias term.

We performed experiments on five classic benchmark datasets from the UCI repository taking 67\% of the data as training pool and the remaining 33\% as a testing set. Note that the weighing-vector-inspired algorithm chooses $4$ points per iterations so we picked the four more uncertain points for the uncertainty sampling algorithm to be fair. 

Figure \ref{fig:al_results} shows average performance curves over 100 runs. The performance from the weighting vector algorithm seems to perform better in four out of the five datasets and slightly worse on the Galaxy dim. and Checkerboard datasets.
\section{Conclusions}
\label{sec:conclusions}
We apply the concepts of metric space magnitude and weighting vector to a wide variety of classical machine learning tasks.  We introduce practical algorithms that are suited to these tasks, and we demonstrate performance that is competitive with, and in many cases, surpasses the performance of benchmark methods.  Additionally, we introduce the notion that the weighting vector can accurately identify boundaries on scattered data that lives in a Euclidean space. 

Prior work in the field of metric space magnitude has generally been theoretical and focused on the magnitude functional itself, and the properties of the weighting vector have been overshadowed. Practical aspects of metric space magnitude and the weighting vector is still an emergent field.  Since magnitude and the weighting vector are well-defined for an extraordinarily wide class of sets, we believe that one natural aim of future work would be to develop vector weighting and magnitude into a robust, unifying foundation for the analysis of familiar, but also highly unusual, spaces.
\bibliography{practical_apps_of_magnitude}
\bibliographystyle{icml2020}
\end{document}